\documentclass[accepted]{article}

\usepackage{graphicx} 
\usepackage{subfigure}
\usepackage{natbib}

\usepackage{algorithm}
\usepackage{algorithmic}

 \usepackage{relsize}
\usepackage{wrapfig}
\usepackage{balance}
\usepackage{array}
\usepackage{amsmath}
\usepackage{amssymb}
\usepackage{amsthm}
\usepackage{todonotes}
\usepackage{dsfont}

\newtheorem{lemma}{Lemma}
\newtheorem{theorem}{Theorem}
\newtheorem{definition}{Definition}
\newtheorem{corollary}{Corollary}

\newcommand{\size}{w}
\renewcommand{\P}{{\mathbb P}}
\newcommand{\R}{I\!\! R}
\newcommand{\NN}{{\mathbb N}}
\newcommand{\calL}{{\cal L}}
\newcommand{\node}[2]{\circ[#1,#2]}
\newcommand{\1}{\mathds{1}}
\newcommand{\E}{\mathbb{E}}
\newcommand{\T}{{\cal T}}
\newcommand{\X}{{\cal X}} 
\newcommand{\eqdef}{\stackrel{{\rm def}}{=}}
\newcommand{\StoSOO}{{\tt StoSOO}}

\DeclareMathOperator*{\argmax}{arg\,max}
\DeclareMathOperator*{\argmin}{arg\,min}

\newcommand{\specialcell}[2][c]{%
  \begin{tabular}[#1]{@{}c@{}}#2\end{tabular}}

\usepackage{hyperref}

\usepackage[accepted]{icml2013} 

\icmltitlerunning{Stochastic Simultaneous Optimistic Optimization}

\begin{document} 

\twocolumn[
\icmltitle{Stochastic Simultaneous Optimistic Optimization}

\vspace*{-0.5em}
\icmlauthor{Michal Valko}{michal.valko@inria.fr\hspace*{-0.3em}}
\vspace*{-0.1em}
 \icmladdress{INRIA Lille - Nord Europe, SequeL team, 40 avenue Halley 59650, Villeneuve d'Ascq, France}
\vspace*{-0.2em}
\icmlauthor{Alexandra Carpentier}{a.carpentier@statslab.cam.ac.uk}
\vspace*{-0.1em}
\icmladdress{Statistical Laboratory, CMS, Wilberforce Road, CB3 0WB, University of Cambridge, United Kingdom}
\vspace*{-0.2em}
\icmlauthor{R\' emi Munos}{remi.munos@inria.fr}
\vspace*{-0.1em}
\icmladdress{INRIA Lille - Nord Europe, SequeL team, 40 avenue Halley 59650, Villeneuve d'Ascq, France}
\vspace*{-0.2em}

\icmlkeywords{StoSOO, optimistic optimization, bandit theory, machine learning, ICML}

\vskip 0.3in
]
\vspace{-1em}
\begin{abstract}
We study the problem of global maximization of a function $f$
given a finite number of evaluations perturbed by noise.
We consider a very weak assumption on the function, namely that it is locally
smooth (in some precise sense)  with respect to some semi-metric, around one
of its global maxima.
Compared to previous works on bandits in general spaces
\cite{kleinberg2008multi,bubeck2011x} our algorithm does not require the
knowledge of this semi-metric.
Our algorithm, \StoSOO{}, follows an optimistic strategy to iteratively
construct
upper confidence bounds over the hierarchical partitions of the function domain
to decide which point to sample next.
A finite-time analysis of \StoSOO{} shows that it performs almost as well as the
best specifically-tuned algorithms even though the local smoothness of the
function
is not known.
\end{abstract}
\section{Introduction}
We consider a function maximization problem of an unknown function $f: \cal X
\to \mathbb{R}$. We assume that every function evaluation
is costly, and therefore we are interested in optimizing the function given a
finite budget of $n$ evaluations. Moreover, the evaluations are perturbed by
noise, i.e.,\ the evaluation of $f$ at a point $x_t\in {\cal X}$ returns a noisy
evaluation $r_t$, assumed to be independent from the previous ones, such that:
\begin{align}\label{ass:noise}
 \E[r_t|x_t]=f(x_t).
\end{align}

One motivation for this setting is a measurement error when dealing with a
stochastic environment.  Another example is the optimization
of some parametric policy operating in a stochastic system.

We assume that there exists at least one global maximizer $x^*\in\X$ of $f$, i.e.\ $f(x^*)=
\sup_{x\in\X} f(x)$.
We aim for an algorithm
which sequentially evaluates $f$ at points
 $x_1, x_2, \dots, x_n$ in the search space $\cal X$
 to find a good approximation to a global maximum.
After $n$ function evaluations the algorithm outputs a point $x(n)$ and its
performance is measured with the loss:
\begin{align}\label{def:regret}
 R_n = \sup_{x \in {\cal X}}(f(x)) - f(x(n))
\end{align}
Our definition of loss is very related
to the \textit{simple regret} in multi-armed bandits~\cite{bubeck2009pure}.
Many algorithms have been developed for this general optimization problem.
However, a lot  of them require some
assumption on the \textit{global} smoothness
of $f$, most typically, they assume a global \textit{Lipschitz}
property
\cite{pinter1995global,strongin2000global,hansen2004global,kearfott1996rigorous,
neumaier2008interval}.
There has been also an interest in designing
sample-efficient strategies, only
requiring \textit{local} smoothness around (one) of the global
maxima~\cite{kleinberg2008multi,bubeck2011x,munos2011optimistic}.
However, these approaches still assume the \textit{knowledge} of this
smoothness, i.e., the metric under which the function is smooth, which may not
be available to the optimizer.

Recently, \citet{munos2011optimistic} proposed the SOO algorithm
for \textit{deterministic} optimization, that
assumes that $f$ is locally smooth with respect to some semi-metric $\ell$, but
that this semi-metric \textit{does not need to be known} to the algorithm. SOO
extends the DIRECT algorithm \cite{jones1993lipschitzian} and other Lipschitz
optimization without the knowledge of the Lipschitz constant
\cite{bubeck2011lipschitz,slivkins2011multi-armed} to the case of any possible
semi-metric by \textit{simultaneously} considering the subspaces that can contain
the optimum.

In this paper, we provide an extension of SOO to the case of noisy evaluations,
which we call Stochastic SOO, or \StoSOO{}. One major difference from SOO is that
we cannot base our exploration strategy only on a single evaluation per cell
since we are dealing with stochastic functions.
Another difference is that we cannot simply return
the highest evaluated point we encountered as $x(n)$
since it is subject to noise.
Our analysis shows that in a large class of
functions (precisely defined in Section~\ref{sec:case.d=0}), the loss of \StoSOO{}
is $\tilde O(n^{-1/2})$, which is of same order as the loss of HOO
\cite{bubeck2011x} or Zooming algorithm \cite{kleinberg2008multi} when using the
best possible metric.

\section{Background}
\label{sec:bck}

\textit{Optimistic optimization} refers to approaches
that implement the \textit{optimism in the face of uncertainty} principle.
This principle became popular in the multi-armed bandit problem~\cite{auer2002finite}
and was later extended to the tree search~\cite{kocsis2006bandit,coquelin2007bandit}
where it is referred to as \textit{hierarchical bandit approach}. The reason is
that a complex problem such as global optimization
of the space $\X$  is treated
as a hierarchy of simple bandit problems.
It is therefore an example of Monte Carlo tree search
which was shown to be empirically successful for
instance in computer Go~\cite{gelly2012grand}.

Optimistic optimization was also used in many other domains, such as
planning~\cite{hren2008optimistic,bubeck2011x}
or Gaussian process optimization~\cite{srinivas2009gaussian}.
This paper applies optimistic approach
to a global \textit{black-box} function optimization.
Table~\ref{tab:hoo} displays  representative
approaches for this setting. The case when the smoothness of the function $f$ is known,
means that the function is either (globally) Lipschitz, weakly Lipschitz or locally Lipschitz around the optimum.
There are numerous algorithms for this setting, the most related to our work are DOO~\cite{munos2011optimistic}
for the deterministic case and Zooming~\cite{kleinberg2008multi} or HOO~\cite{bubeck2011x}
for the stochastic one\footnote{Note that the loss~\eqref{def:regret} considered
here is different but related to the usual cumulative regret defined in the
bandit setting, see e.g.~\cite{bubeck2009pure}.}.
This setting has been also considered in a Bayesian framework,
in particular the \emph{expected-improvement} strategy~\cite{osborne2010bayesian}
which was theoretically analyzed when the assumption of smoothness is
data-driven~\cite{bull2011convergence}.

One of the disadvantages of these algorithms is that
however strong or mild are the assumptions on $f$,
the quantities that express them (i.e.\ a prior, a Lipschitz constant, or a
semi-metric in DOO) need to be \textit{known} to the algorithm. On the other
hand, for the case of deterministic functions there exist approaches that do not
require this knowledge, such as DIRECT or SOO.

However, neither DIRECT nor SOO can deal with
stochastic functions. Therefore, we extend the SOO algorithm to the stochastic setting
and provide a finite-time analysis of its performance.

\begin{table}
\begin{center}
\caption{Hierarchical optimistic optimization algorithms}
\label{tab:hoo}
\def\arraystretch{1.5}
 \begin{tabular}[r]{|r|cc|} \hline
 & \textbf{deterministic} & \textbf{stochastic} \\ \hline
\specialcell{\textbf{known} \\[-0.8em] smoothness} & DOO & Zooming
or HOO \\
\hline
\specialcell{\textbf{unknown} \\[-0.8em] smoothness} & DIRECT or SOO &
\specialcell{ \textbf{\StoSOO{}} \\[-1em] \begin{small}this paper \end{small}} \\
\hline
\end{tabular}
\end{center}
\end{table}
\section{Algorithm}
\label{sec:algo}
\StoSOO{} is a tree-search based algorithm that iteratively constructs
finer and finer partition of the search space ${\cal X}$.
The partitions are represented as nodes of a $K$-ary tree $\T$ and the nodes
are organized by their depths $h \geq 0$, with $h=0$ being
the root node, and indexed by $1\leq i\leq K^h$. We denote $\node{h}{i}$, the
$i$-th node
at depth $h$. Each of the nodes $\node{h}{i}$ corresponds to a cell
$\X_{h,i}\subseteq {\cal X}$ in the
partitioning, i.e.,~to a subset of $\X$ with an associated representative point
 $x_{h,i}\in \X_{h,i}$.
\subsection{Assumptions}
We now state our main assumption, which is also used in
SOO~\cite{munos2011optimistic}.
The first part of the assumption is about the existence of a semi-metric $\ell$
such that the function $f$ is locally smooth with respect to it. We stress that
although it quantifies the smoothness of $f$,
it {\bf only requires the existence} of $\ell$ and \textit{not} the
\textit{knowledge} of it. For illustrative examples and
discussion on this part we refer the
reader to~\cite{munos2011optimistic}.
The second part is about the structure of the
hierarchical partitioning with respect to $\ell$.
This partitioning is fixed and given to the algorithm
as a parameter.


\paragraph{Assumption}
There exists a semi-metric $\ell:\X\times\X\rightarrow\R^+$ (i.e.\ for
$x,y\in\X$, we have
$\ell(x,y)=\ell(y,x)$ and $\ell(x,y)=0$ if and only if
$x=y$) such that:
\begin{itemize}
\label{ass:A1-2}
\item [{\bf A1} \label{ass:A1}]
 
{\bf (local smoothness of $f$):} For all $x\in\X$:
\begin{equation}
\label{ass:f}
f(x^*)-f(x) \leq \ell(x,x^*).
\end{equation}
\item [{\bf A2} \label{ass:A2}]  {\bf (bounded diameters and well-shaped
cells):}
There exists a decreasing sequence $\size(h)>0$, such that for any depth
$h\geq0$ and  for any cell $\X_{h,i}$ of depth $h$, we have
$\sup_{x\in X_{h,i}} \ell(x_{h,i},x)\leq \size(h).$
Moreover, there exists $\nu >0$ such that for any depth $h\geq 0$, any cell
$\X_{h,i}$ contains a $\ell$-ball of radius $\nu \size(h)$ centered in
$x_{h,i}$.
\end{itemize}

Assumption \hyperref[ass:A1]{A1} guarantees that $f$ does not
decrease too fast around
one global optimum $x^*$. This can be thought of as a one-sided local Lipschitz
assumption.
Note that although we require that \eqref{ass:f} is satisfied
for all $x\in \X$, this assumption essentially sets constraints to the function
$f$ locally around $x^*$, since when $x$ is such that
$\ell(x,x^*)>
\sup f - \inf f$, then the assumption is automatically satisfied. Thus when
this property holds, we say that {\bf $f$ is locally smooth with
respect to~$\ell$ around
its maximum}.

Assumption \hyperref[ass:A2]{A2} assures the regularity
of the partitioning, in particular that the size of the cells decreases
with their depths and that their shape is not skewed in some dimensions.

\subsection{Stochastic SOO}

Algorithm~\ref{alg:StoSOO} displays
the pseudo-code of the \StoSOO{} algorithm. The algorithm operates in the
traversals
of the tree starting from the root down to the
current depth$(\T)$, that is upper bounded by $h_{\max}$, a parameter of the
algorithm.
During each traversal (a whole pass of the ``for'' cycle) \StoSOO{} selects
a set of promising nodes, at most one per depth $h$. These nodes are then
either \textit{evaluated} or \textit{expanded}.

\begin{algorithm}[ht]
\begin{algorithmic}
\STATE {\bf Parameters:} number of function evaluations $n$, maximum number of
evaluations per node $k>0$, maximum depth $h_{\max}$, and $\delta>0$.
\STATE \textbf{Initialization:}
\STATE \quad $\T \gets \{\node{0}{0}\}$ \COMMENT{root node}
\STATE \quad   $t \gets  0$ \COMMENT{number of evaluations}
\WHILE{$t\leq n $}
\STATE $b_{\max} \gets -\infty$
\FOR{$h=0$ to $\min($depth$(\T), h_{\max})$}
  \IF{$t\leq n$}
  \STATE For each leaf $\node{h}{j} \in \calL$, compute its $b$-value:
  \STATE \quad $b_{h,j}(t) = \hat\mu_{h,j}(t) + \sqrt{\log(nk/\delta)/(2
T_{h,j}(t))}$
  \STATE Among leaves $\node{h}{j}\in\calL_t$ at depth $h$, select
$$\node{h}{i}\in\argmax_{\node{h}{j}\in\calL} b_{h,j}(t)$$
  \IF{$b_{h,i}(t) \geq b_{\max}$}
     \IF{$T_{h,i}(t) < k$}
     \STATE Evaluate (sample) state $x_t=x_{h,i}$.
      \STATE Collect reward $r_t$ (s.t.~$\E[r_t|x_t]=f(x_t)$).
     \STATE  $t\gets t+1$
     \ELSE[i.e.\ $T_{h,i}(t)\geq k$, expand this node]
       \STATE Add  the $K$ children of $\node{h}{i}$ to $\T$
       \STATE $b_{\max} \gets b_{h,i}(t)$
     \ENDIF
    \ENDIF
  \ENDIF
\ENDFOR
\ENDWHILE
\STATE \textbf{Output:} The representative point with the
highest $\hat\mu_{h,j}(n)$ among the deepest expanded nodes:
$$x(n) = \argmax_{x_{h,j}} \hat\mu_{h,j}(n) \mbox{ s.t. } h =
\mbox{depth}(\T\setminus \calL).$$
\end{algorithmic}
\caption{\StoSOO{} \\ \hspace*{0.55cm}
\textit{Stochastic Simultaneous Optimistic Optimization}
\label{alg:StoSOO}}
\end{algorithm}

Evaluating a node at time $t$
means sampling the
function in the representative point $x_{h,i}$ of the cell $\X_{h,i}$ and observing the
evaluation $r_t$ according to \eqref{ass:noise}.
Expanding a node $\node{h}{i}$, means splitting its corresponding
cell into its $K$ sub-cells corresponding to the children:
$$\{\node{h+1}{i_1},\node{h+1}{i_2},\dots,\node{h+1}{i_K}\}.$$
We denote by $\calL$ the set of leaves in $\T$, i.e.\ the nodes with no
children.
At any time, only the leaves are eligible for an evaluation or expansion
and we never expand the leaves beyond  depth $h_{\max}$.
If the function $f$ were deterministic, such as in
SOO~\cite{munos2011optimistic}, we would expand (simultaneously) any leaf
$\node{h}{i}$
whose value $f(x_{h,i})$ is the largest among all leaves of the same or a lower
depth.
The reason for this choice is that by
Assumption~\hyperref[ass:A1]{A1} all such nodes may contain $x^*$.
Unfortunately, we do not receive
$f(x_{h,i})$,
but only a noisy estimate $r_t$.
Therefore, the main algorithmic idea of \StoSOO{} is
to evaluate the leaves several times in order to build a confident
estimate of $f(x_{h,i})$. For this purpose,
let us define $\hat\mu_{h,i}(t)=\frac{1}{T_{h,i}(t)}\sum_{s=1}^t r_s
 \1\{x_s\in
\X_{h,i}\}$ the empirical average of rewards obtained at state $x_{h,i}$ at time
$t$, where $T_{h,i}(t)$ is the number of
 times that $\node{h}{i}$ has been sampled up to time $t$.

\StoSOO{} builds an accurate estimate of $f(x_{h,i})$ before $\node{h}{i}$ is
expanded.
To achieve this, we define an upper confidence bound (or a $b$-value) for
each node $\node{h}{i}$ as:
\begin{align}\label{eq:b-val.SSOO}
b_{h,i}(t)\eqdef \hat\mu_{h,i}(t) + \sqrt{\frac{\log(nk/\delta)}{2
T_{h,i}(t)}},
\end{align}
where $\delta$ is the confidence parameter.
In the case of $T_{h,i}(t) = 0$, we let $b_{h,i}(t)= \infty$.
 We refer to $\sqrt{\log(nk/\delta)/2 T_{h,i}(t)}$
 as to the \emph{width} of the estimate.
Now instead of selecting the promising nodes
according to their values $f(x_{h,i})$ we  select them according to their
$b$-values $b_{h,i}$.

Our algorithm is \textit{optimistic} since it considers such leaves for the
selection
whose $b$-value is \textit{maximal} among leaves at  depth $h$ or lower
depths,
since those leaves are likely
to contain the optimum $x^*$ at time $t$, given the observed samples and
Assumption~\hyperref[ass:A1]{A1} on~$f$.

The important question is now how many times should we
evaluate the node before we decide to expand it.
Again, if we knew the semi-metric $\ell$ we would be able to calculate
the appropriate count for each depth $h$.
Since we do not know it, we instead
evaluate each node a fixed number of $k$ times before its expansion.
We address the setting of $k$, $h_{\max}$,  and $\delta$ in
Sections~\ref{sec:analysis} and~\ref{sec:case.d=0}.
Our analysis shows that under appropriate assumptions on $f$
(discussed in
Section~\ref{sec:case.d=0}) we can bound the expected regret as $\E{[R_n]} =
O\left(\log^2(n)/\sqrt{n}\right)$ by setting $k = n/\log^3(n)$, $h_{\max} =
\sqrt{n/k}$, and $\delta = 1/\sqrt{n}$.

In the algorithm, we keep track of the number of evaluations $t$ in order to
finish
when it reaches $n$, the maximum number of evaluations, i.e., the budget.
Since we are facing a stochastic setting, we cannot simply
output the value that received the highest reward during $n$ evaluations, as
it is the case in most of the deterministic approaches.
Instead, we return the representative point $x_{h,j}$ of the node with the highest estimate
$\hat\mu_{h,j}(n)$ among the deepest expanded nodes, i.e.,\ such that
$h=\mbox{depth}(\T\setminus \calL)$.

\section{Analysis}
\label{sec:analysis}

In this section we analyze the performance
of \StoSOO{} and upper bound the loss \eqref{def:regret}
as a function of the number of evaluations.
We assume that the rewards are bounded\footnote{The analysis can
be easily extended to the case when the noise is sub-Gaussian.}
 by $|r_t|
\leq 1 $ for any $t$.
In order to derive a loss bound we define a measure of the quantity of
near-optimal states, called {\em near-optimality dimension}. This measure is
closely related to similar measures~\cite{kleinberg2008multi,bubeck2008online}.
For
any $\varepsilon>0$, let us write the set of $\varepsilon$-optimal states as:
$$\X_\varepsilon \eqdef \{ x\in\X, f(x)\geq f^* - \varepsilon\}.$$

\begin{definition}\label{def:near-opt}
The $\nu$-near-optimality dimension is the smallest $d\geq 0$ such that there
exists $C>0$ such that for any $\varepsilon>0$, the maximum number of disjoint
$\ell$-balls of radius $\nu \varepsilon$ and center in $\X_\varepsilon$ is less
than
$C\varepsilon^{-d}$.
\end{definition}
\StoSOO{} maintains the upper confidence bounds ($b$-values) for each cell in
order to decide which cell to sample or expand. We start by quantifying 
the probability that all the average estimates $\hat\mu_{h,j}(t)$ are at any
time $t$ within those $b_{h,j}(t)$-values. For this purpose
we define the event in which this occurs and then show that this event
happens with high probability.

\begin{lemma}\label{lemma:xi}
Let $\xi$ be the event under which all average estimates
are within their widths:
\begin{align*}
\xi\eqdef\Big\{ &\forall h,i,t \mbox{ s.t. } h\geq 0, 1\leq i<K^h,
1\leq t\leq n, \mbox{and}
\notag\\
& T_{h,j}(t) > 0: \big| \hat\mu_{h,j}(t) - f(x_{h,j}) \big| \leq
\sqrt{\frac{\log(nk/\delta)}{2T_{h,j}(t)}} \Big\},
\end{align*}
then $\P(\xi)\geq 1-\delta$.
\end{lemma}
\begin{proof}
Let $m$ denote the (random) number of different nodes sampled by the algorithm
up to time $n$.
Let $\tau_i^1$ be the first time when the $i$-th new node $\node{H_i}{J_i}$ is
sampled, i.e., at time $\tau_i^1-1$ there are only $i-1$ different nodes that
have been sampled whereas at time $\tau_i^1$, the $i$-th new node
$\node{H_i}{J_i}$  is sampled for the first time. Let $\tau_i^s$, for $1\leq
s\leq T_{H_i,J_i}(n)$, be the time when the node $\node{H_i}{J_i}$ is sampled
for the $s$-th time. Moreover, we denote $Y_i^s = r_{\tau_i^s} -
f(x_{H_i,J_i})$. Using this notation, we rewrite $\xi$ as:
\begin{align}\label{eq:xi}
\xi = \Bigg\{ &\forall i,u \mbox{ s.t. }, 1\leq i \leq m,
1\leq u\leq T_{H_i, J_i}(n),
\notag\\
& \bigg| \frac{1}{u}\sum_{s=1}^u Y_s^i \bigg| \leq \sqrt{\frac{\log(n k/\delta)}{2u}} \Bigg\}.
\end{align}
Now, for any $i$ and $u$, the $(Y_i^s)_{1\leq s\leq u}$ are i.i.d.~from
some distribution $\nu_{H_i,J_i}$. The node $\node{H_i}{J_i}$ is random and
depends on the past samples (before time $\tau_i^1$) but the  $(Y_i^s)_s$ are
conditionally independent given this node and consequently:
\begin{align*}
\P&\Bigg( \bigg| \frac{1}{u}\sum_{s=1}^u Y_s^i \bigg| \leq \sqrt{\frac{\log(n
k/\delta)}{2u}}\Bigg) = \\
& =   \mathlarger{\E}_{\node{H_i}{J_i}}\, \P\bigg( \bigg|
\frac{1}{u}\sum_{s=1}^u Y_s^i \bigg| \leq \sqrt{\frac{\log(n k/\delta)}{2u }}
\ \Bigg| \node{H_i}{J_i} \bigg) \\
& \geq 1-\frac{\delta}{nk},
\end{align*}
using Chernoff-Hoeffding's inequality. We finish the proof
by taking a union bound over all values of $1\leq i\leq n$ and $1\leq u\leq k$.
\end{proof}
Lemma~\ref{lemma:xi} shows that when the leaf is expanded then
with high probability the mean estimate $\hat\mu_{h,j}(t)$
is very close to its true value. 
Specifically, when the node is expanded then 
with probability $1-\delta$ uniformly for all $h,j,$ and $t$, we have that:
\begin{align}\label{def:epsilon}
 |\hat\mu_{h,j}(t)-f(x_{h,j})|\leq \varepsilon,
\end{align}

where $\varepsilon = \sqrt{\log(nk/\delta)/2 k}$.
We use this lemma to 
show that the expanded nodes are with high probability 
close to optimal.


\begin{definition}\label{def:expansion-set}
Let the expansion set at  depth $h$ be the set of all
nodes that could be potentially expanded before the optimal
node at  depth $h$ is expanded:\footnote{The reason for such definition will become
apparent in the proof of Lemma~\ref{lem:depth.stosoo}.}
\[
I_h^\varepsilon \eqdef \{ \mbox{nodes } \node{h}{i} \mbox{ such that }
  f(x_{h,i}) + \size(h) + 2\varepsilon \geq f^* \}.
\]
\end{definition}
Recall that even though this definition uses $\size(h)$
that depends on the unknown metric $\ell$, the \StoSOO{} algorithm does not need to
know it. Now, let us denote $h^*_t$ the deepest depth of the
expanded node at time $t$, that contains the optimum $x^*$.
Notice that in general the algorithm may have at  time $t$
also expanded  some (suboptimal) nodes in the deeper depths.
In the following, we show that they are not too many of these.
Specifically, for each depth $h$, we lower bound the number 
of evaluations after which the $h^*_t$ needs to be at least~$h$.

\begin{lemma}\label{lem:depth.stosoo}
Let depth $h \in \{0,h_{\max}\}$ be any depth and: 
$$t_{h} \eqdef (k + 1)h_{\max}(|I_0^\varepsilon| + |I_1^\varepsilon| + \dots +
|I_h^\varepsilon|).$$
After we evaluated at least $t
\ge t_h$ nodes, then in the event $\xi$,
 the depth $h^*_t$ of the deepest node in  the optimal
branch is at least $h$, i.e.,\ $h^*_t \geq h$.
\end{lemma}
\begin{proof}
\textit{By induction on $h$.}
For
$h=0$, the
lemma holds trivially since $h^*_t \geq 0$.
For the induction step, let us assume that the lemma holds for all
$h \in \{0, \dots,  h'\}$, where $h' < h_{\max}$
and we are to show it holds for $h'+1$ as well.
Assume we have already evaluated $t_{h'+1}$ nodes, i.e.~that we are at time
$t \geq t_{h'+1}$.
Since $t_h$ is increasing in $h$, we have also evaluated $t_{h'}$ nodes
and $h^*_t \geq h'$ from the induction step.
That means that the optimal branch is expanded at least up to the depth $h'$.
Now consider any node $\node{h'+1}{i}$ at depth $h'+1$, 
that was expanded. If it was expanded before the
optimal node $\node{h'+1}{i^*}$ at depth $h'+1$ was expanded,
then $b_{h^*_t+1,i}(t) \geq b_{h^*_t+1,i^*}(t)$.
According to Lemma~\ref{lemma:xi}, the average estimates
$\hat \mu_{h,j}(t)$ are at most $\varepsilon$ away
from their true values, with $\varepsilon$ defined in~$\eqref{def:epsilon}$.
Therefore in the event $\xi$, the true
values of the expanded and the optimal
node are at most $2\varepsilon$ apart:
\begin{align}\label{eq:2e_bound}
f(x_{h^*_t+1,i}) \geq f(x_{h^*_t+1,i^*})-2\varepsilon.
\end{align}
Since the node $\node{h^*_t+1}{i^*}$ contains the optimum $x^*$,
then by Assumptions~\hyperref[ass:A1-2]{A1-2}, we get:
$$f(\!x_{h^*_t+1,i^*}\!) + \size(\!h+1\!)\!\ge\! f(\!x_{h^*_t+1,i^*}\!) +
\ell(x_{h^*_t+1,i^*},x^*)\!\ge\! f^*\!.$$
Combining this with~\eqref{eq:2e_bound}, we obtain that:
$$f(x_{h^*_t+1,i}) \geq f(x_{h^*_t+1,i^*})-2\varepsilon \geq f^* - \big[
\size(h^*_t+1)+2\varepsilon\big].$$
This means that all the nodes  $\node{h'+1}{i}$
expanded before $\node{h'+1}{i^*}$ are 
$[\size(h^*_t+1)+2\varepsilon]$-optimal.
By Definition~\ref{def:expansion-set},
there are exactly $|I_{h'+1}^\varepsilon|$
such nodes. Each traversal of the tree
in the \StoSOO{} algorithm 
selects one of these nodes for evaluation.
Since $k$ evaluations are required before the expansion,
after $(k+1)|I_{h'+1}^\varepsilon|$ traversals,
$\node{h'+1}{i^*}$ must have been expanded.
To guarantee this many traversals,
we need $(k+1) h_{\max}|I_{h'+1}^\varepsilon|$
evaluations after $t_h'$ previous evaluations.
This is equal to $t_{h'+1}$ and thus $h^*_t \geq h' + 1$.
\end{proof}
Lemma~\ref{lem:depth.stosoo} bounds the number
of needed evaluations in the terms of the expansion set sizes to assure that
the optimal
node was expanded.
Naturally, we would like to know, how big these expansion sets can be.
The following lemma upper bounds the size of expansion sets
up to  depth where $\size(h)$ is of the order of $\varepsilon$.
For this purpose, we define $h_\varepsilon$ as:
\begin{equation}\label{def:h_epsilon}
h_\varepsilon = \argmin \{h \in \NN: \size(h+1) < \varepsilon\}.
\end{equation}

\begin{lemma}\label{lemma:expansion-set-size}
Let $d$ be a $\nu/3$-near-optimality dimension
and $C$ the related constant.
Then for each $h \leq h_\varepsilon $, the cardinality of
the expansion set at  depth $h$ is in the event $\xi$ bounded as:
\[
|I_h^\varepsilon| \le C\left(\size\left(h\right) + 2\varepsilon\right)^{-d}.
\]
\end{lemma}

\begin{proof}
\textit{By contradiction.} Assume that for some $h \leq h_\varepsilon$,
$|I_h^\varepsilon| > C\left(\size\left(h\right) + 2\varepsilon\right)^{-d}$.
By definition of $|I_h^\varepsilon|$,  each representative point
$x_{h,i}$ of the node $\node{h}{i}$ is [$\size(h) + 2\varepsilon$]-optimal. By
Assumption~\hyperref[ass:A2]{A2}, each cell associated with the node
$\node{h}{i}$ at
depth $h$ contains  a ball of radius
$\nu\size(h) = \frac{\nu}{3}\cdot3\size(h) \ge \frac{\nu}{3} (\size(h) +
2\varepsilon)$ with the representative point $x_{h,i}$, because for $h \leq h_\varepsilon$,
we have that $ \varepsilon \leq \size(h)$ by~\eqref{def:h_epsilon}.
Since the cells are disjoint,
we have a contradiction with $\nu/3$-near-optimality dimension being~$d$.
\end{proof}
We now link the depth of the tree after $n$ iterations with the
loss as defined in~\eqref{def:regret}.
\begin{theorem}\label{thm:stosoo-regret}
Assume that Assumptions~\hyperref[ass:A1-2]{A1-2} hold. Let $d$ be the
$\nu/3$-near-optimality dimension and $C$ be the corresponding constant. 
Then the loss of \StoSOO{} run with parameters $k$, $h_{\max}$, and $\delta>0$, after $n$ iterations
is bounded, with probability $1-\delta$, as:
$$R_n \le 2\varepsilon + \size\left(\min\left( h(n) -1, h_\varepsilon, h_{\max}
\right)\right)$$
where $\varepsilon=\sqrt{\log(nk/\delta)/(2k)}$ and $h(n)$ is the smallest $h \in
\NN$, such that:
$$C (k +1) h_{\max}
 \sum_{l=0}^{h} \left(\size\left(l\right)+2\varepsilon\right)^{-d} \ge n.
$$

\end{theorem}
\begin{proof}
Let us first consider the case when $h(n)-1 \leq h_\varepsilon$.
Then we can use Lemma~\ref{lemma:expansion-set-size} to show that:
\begin{align}\label{eq:number-evaluated-bound}
n &> C (k + 1) h_{\max}
\sum_{l=0}^{h(n)-1}
\left(\size\left(l\right)+2\varepsilon\right)^{-d} 
\nonumber \\
&\ge  (k + 1) h_{\max}
\sum_{l=0}^{h(n)-1} |I_l^\varepsilon| =  t_{h(n)-1}
\end{align}
If $h(n) - 1 \le h_{\max}$ then by Lemma~\ref{lem:depth.stosoo},
 $h_n^*\ge h(n)-1$. If, however, $h(n) - 1 > h_{\max}$,
then by~\eqref{eq:number-evaluated-bound} the algorithm has expanded
all potentially optimal nodes on the level $h_{\max}$
 and therefore $h_n^* \ge h_{\max}$. Nonetheless
the algorithm does not go beyond $h_{\max}$, so necessarily
$h_n^* =  h_{\max}$. Hence, in the case when $h(n)-1 \le h_\varepsilon$,
$h_n^* \ge \min\{h(n) - 1,  h_{\max}\}$.
Now consider the opposite case, i.e.,\ when $h(n)- 1 \ge h_\varepsilon + 1$.
We can now use Lemma~\ref{lemma:expansion-set-size}, but only up to  depth
$h_\varepsilon$, to get that $n > t_{h_\varepsilon}$.
Similarly to the previous case, we deduce that $h_n^* \ge \min\{h_\varepsilon,
h_{\max}\}$. Altogether, $h_n^* \ge \min\{h(n) - 1 ,h_\varepsilon,
h_{\max}\}$. Let $\node{h}{j}$ be the deepest node that has been
expanded after $n$ evaluations. We know that $h \ge h_n^*$.
Let also $\node{h_n^*}{i^*}$ be the optimal node at the
depth $h_n^*$.
As $\node{h}{j}$ was expanded, the true value of its representative point
and the representative point of $\node{h_n^*}{i^*}$
is in the event $\xi$ at most $2\varepsilon$ away
and therefore we conclude that:
\begin{align}
f(x_{h,j}) &\geq f(x_{h^*_n,i^*})-2\varepsilon \geq f^* -
[\size(h^*_n)+2\varepsilon] \nonumber\\\nonumber
&\ge f^* - [ \size(\min\{h(n) - 1 ,h_\varepsilon,
h_{\max}\})+2\varepsilon].
\end{align}
\end{proof}

\section{The important case $d=0$}\label{sec:case.d=0}
We now deduce the following corollaries for the case when the near-optimality dimension $d=0$ and the diameters $\size(h)$ are exponentially decreasing. We postpone the discussion about this important case $d=0$ to Section~\ref{sec:case.d=0.intuition}.

\begin{corollary}\label{thm:stosoo}
Assume that the diameters of the cells decrease exponentially fast,
i.e.,\ $\size(h)=c \gamma^h$ for some $c>0$ and $\gamma<1$. Assume that the
$\nu/3$-near-optimality dimension is $d=0$ and let $C$ be the corresponding
constant. Then the expected loss of \StoSOO{} run with parameters $k$,
$h_{\max}=\sqrt{n/k}$, and $\delta>0$, is bounded as:
\begin{equation}
\label{eq:stosoo.bound}
\E[R_n] \leq (2 + 1/\gamma)\varepsilon + c \gamma^{\sqrt{n/k} \min\{0.5/C,1\} - 2} + 2\delta.
\end{equation}
\end{corollary}
\begin{proof}
When $d=0$, then $\big[\size(l)+2\varepsilon\big]^{-d} = 1$ and
by definition of $h(n)$,  we have that
$n \leq C (k + 1) h_{\max}
\sum_{l=0}^{h(n)} \big[\size(l)+2\varepsilon\big]^{-d}
= C (k + 1) h_{\max}
(h(n)+1)$, which implies $h(n) \ge n/(C(k+1)h_{\max}) -1$.
Intuitively, the deeper is the node we return, the lower regret we can
incur. This suggests the choice of $h_{\max} = \sqrt{n/k}$,
in which case we get $h(n) \ge \sqrt{n}/(2C\sqrt{k}) -1$, since $k\ge1$.
Moreover, since $\size(h)=c \gamma^h$, then by definition of $h_\varepsilon$
we have  that: 
$$\size(h_\varepsilon) = c\gamma^{h_\varepsilon + 1}/\gamma  = \size(h_\varepsilon + 1)/\gamma < \varepsilon/\gamma.$$
By Theorem~\ref{thm:stosoo-regret},
we have that in the event $\xi$, the regret of \StoSOO{} is at most:
\begin{align*}
R_n&\leq
2\varepsilon + \size(\min\{h(n)-1,h_\varepsilon,h_{\max}\}) \\
&\leq 2\varepsilon + \size(h_\varepsilon) + \size(\min\{h(n)-1,h_{\max}\}) \\
&\leq (2 + 1/\gamma)\varepsilon + c \gamma^{\sqrt{n/k} \min\{0.5/C,1\} - 2}
\end{align*}

%
%
We obtain the upper bound on the expected loss~\eqref{eq:stosoo.bound}, by
considering that by Lemma~\ref{lemma:xi}, $\xi$ holds with probability
$1-\delta$
and $|r_t| \leq 1$.
\end{proof}
\begin{corollary}\label{thm:col1}
For the choice $k=n/\log^3(n)$  and $\delta =
1/\sqrt{n}$, we have:
$$\E[R_n]=O\Big(\frac{\log^2(n)}{\sqrt{n}}\Big).$$
\end{corollary}
This result shows that, surprisingly, \StoSOO{} achieves the same rate $\tilde O(n^{-1/2})$, up to a logarithmic factor, as the HOO algorithm run with the best possible metric, although \StoSOO{} does not require the knowledge of~it.
\begin{proof}
Setting $k=n/\log^3(n)$ and $\delta = 1/\sqrt{n}$
we can upper bound $\varepsilon$ in~\eqref{eq:stosoo.bound} which was defined
in~\eqref{def:epsilon} as:
$$\varepsilon\!=\! \sqrt\frac{\log(nk/\delta)}{2 k}
\!=\!\sqrt\frac{\log(nk\sqrt{n})\log^3(n)}{2 n}
\!\leq\! \sqrt{\frac 54} \frac{\log^2(n)}{\sqrt{n}}. $$
Now for $n$ bigger than a quantity exponential in $C/\log(1/\gamma)$,
the second term in~\eqref{eq:stosoo.bound} becomes negligible 
and the upper bound for this choice follows.
\end{proof}

\subsection{Some intuition about the case $d=0$}\label{sec:case.d=0.intuition}

We have seen that the near-optimality dimension $d$ is a property of both the
function and the semi-metric $\ell$. Since \StoSOO{} does not require the knowledge
of the semi-metric $\ell$ (it is only used in the analysis),
one can choose the best possible semi-metric $\ell$, {\bf possibly according to
the function $f$ itself}, in order to have the lowest possible value of $d$. The
case $d=0$ thus corresponds to the following assumption on $f$: there exists a
semi-metric $\ell$ such that:
\textbf{1)} $f$ is locally smooth w.r.t.~$\ell$ around a global optimum $x^*$
(i.e.\ such that \eqref{ass:f} holds)
 \textbf{2)} the diameters of the cells (measured with $\ell$) decrease exponentially fast, and
 \textbf{3)} there exists $C>0$ such that for any $\varepsilon>0$, the maximal number of disjoint $\ell$-balls of radius $\nu\varepsilon/3$ centered in
$\X_\varepsilon$ is less than $C$ (i.e.\ the near-optimality dimension $d$ is $0$).

\subsection{Examples}\label{sec:case.d=0.examples}
Let us consider the case of functions $f$ defined on $[0,1]^D$ that are locally equivalent to a
polynomial of degree $\alpha$ around their maximum,
i.e.,\ $f(x)-f(x^*)=\Theta(\|x-x^* \|^\alpha)$
for some $\alpha>0$, where $\|\cdot\|$ is any norm. 
The choice of semi-metric $\ell(x,y) =  \|x-y \|^\alpha$ implies that the
near-optimality dimension $d=0$. This covers already a large class of functions
(such as the functions plotted in Figure~\ref{fig:d0functions}: the \emph{two-sine
product} function for which $\alpha=2$ and the non-Lipschitz \emph{garland} function
for which $\alpha=1/2$).
 \begin{figure}
 \begin{center}
\includegraphics[width=0.45\columnwidth]{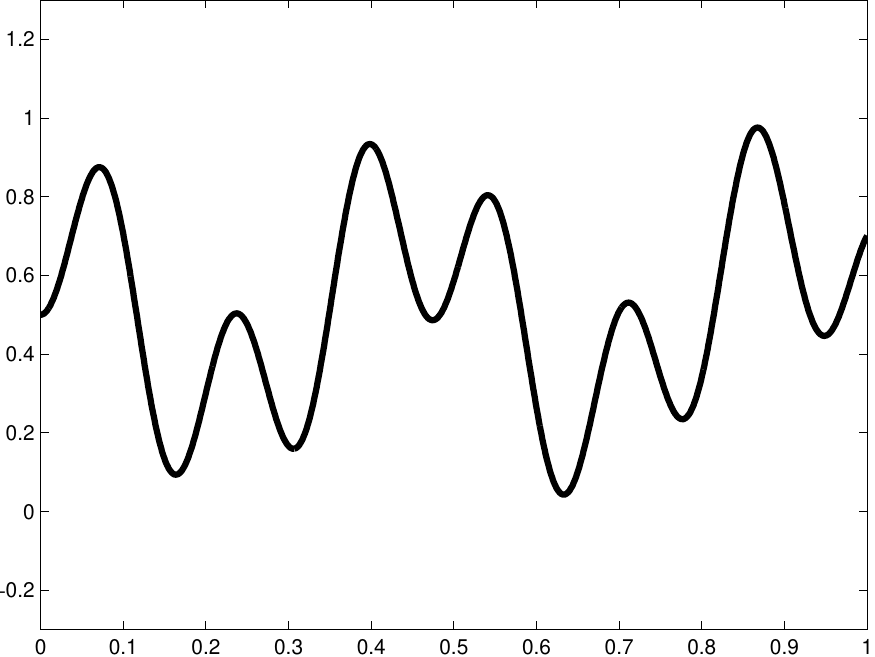}
\quad
 \includegraphics[width=0.45\columnwidth]{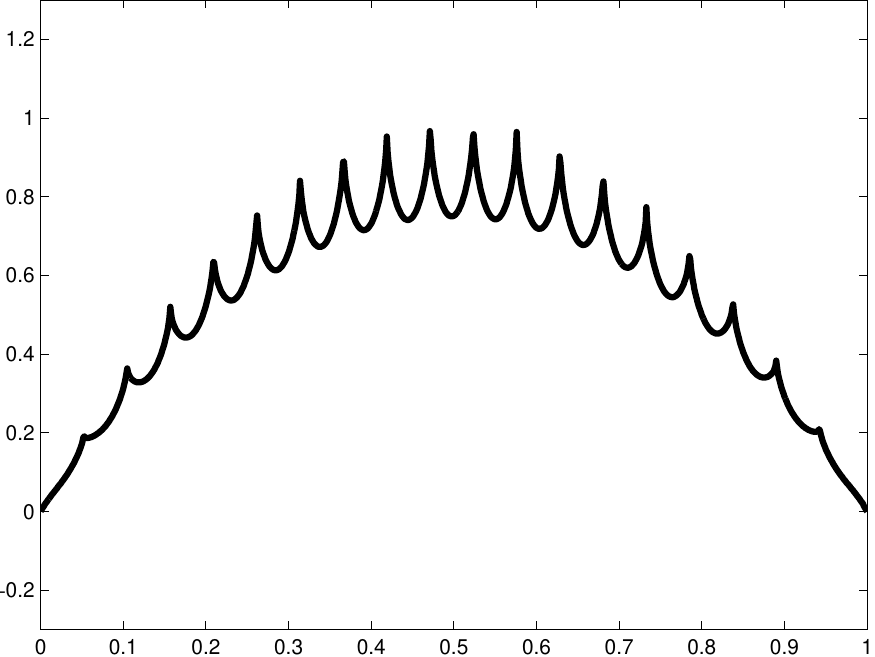}
  \caption{Functions with $d=0$.
\textbf{Left:} \emph{Two-sine product} function
$f_1(x) = \tfrac12 (\sin(13x) \cdot \sin(27x)) + 0.5$.
\textbf{Right:} \emph{Garland} function:
 $f_2(x) = 4 x (1-x) \cdot (\tfrac34+\tfrac14(1-\sqrt{|\sin(60x)|}))$.}
\label{fig:d0functions}
\end{center}
 \end{figure}

More generally, we consider a finite dimensional and bounded space,
i.e., such that ${\cal X}$ can be packed by $C_{\cal X}\varepsilon^{-D}$
$\ell$-balls with radius $\varepsilon$ (e.g., Euclidean space $[0,1]^D$)
and such that ${\cal X}$ has a finite doubling constant
(defined as minimum value $q$ such that every ball in ${\cal X}$ can be packed by at most $q$ balls
in ${\cal X}$ of half the radius).
Let a function in such space have upper- and lower envelope
around $x^*$ of the same order (Figure~\ref{fig:near-opt-example1}), i.e.,\
there exists constants $c \in (0,1)$, and
$\eta>0$, such that for all $x\in\X$:
\begin{equation}\label{eq:enveloppes}
\min(\eta, c \ell(x,x^*))\leq f(x^*)-f(x) \leq \ell(x,x^*).
\end{equation}
We show that all such functions have a near-optimality dimension $d=0$
according to Definition~\ref{def:near-opt}, (where $\nu = 1$ for simplicity),
which means that for all $\varepsilon > 0$,
the packing number of ${\cal X}_\varepsilon$ is upper-bounded by a constant.

\begin{figure}[ht]
\begin{center}
\includegraphics[width=0.9\columnwidth]{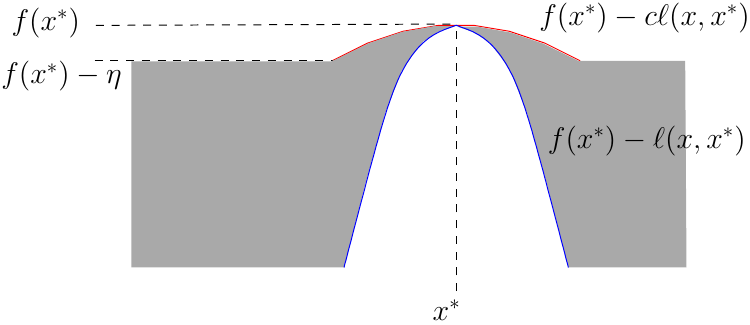}
\caption{Any function satisfying \eqref{eq:enveloppes} lies in the gray
area and possesses a lower- and upper-envelopes that are of same order around $x^*$. }
\label{fig:near-opt-example1}
\end{center}
\end{figure}

In the case when $\varepsilon < \eta$, due the upper envelope
we have that:
${\cal X}_\varepsilon \subset \{ x: c \ell(x,x^*) \leq \varepsilon \},$
which corresponds to an $\ell$-ball centered in $x^*$ with radius $\varepsilon/c$.
This ball can be packed by no more than a  constant
number of $C'$ $\ell$-balls of radius $\varepsilon$.
$C'$ is necessarily finite because the doubling constant $q$
is finite.  For example in $[0,1]^D,$ if $\ell(x,y) =  \|x-y \|_\infty$ then
$C' = (1/c)^D$.

In the opposite case when $\varepsilon \ge \eta$,
the radius of disjoint $\ell$-balls 
that could possibly pack ${\cal X}_\varepsilon$
is at least $\eta$.
Noting that ${\cal X}_\varepsilon \subset {\cal X}$,
we can upper bound the packing number of the whole space ${\cal X}$,
by a constant $C_{\cal X}(\eta)^{-D}$ that is independent of $\varepsilon$.
Finally, defining $C = \max \{C',C_{\cal X}(\eta)^{-D}\}$
we have that for all $\varepsilon$, the maximum number of disjoint $\ell$-balls of radius $\varepsilon$
and center in  ${\cal X}_\varepsilon$ is less than a $C$ and therefore $d=0$.

Even more generally, one can even define the semi-metric $\ell$ according to the behavior of $f$ around $x^*$ in order that  \eqref{ass:f} holds. For example if the space $\X$ is a normed space (with norm $\|\cdot\|$), one can define the metric $\ell(x,y)\eqdef\tilde\ell(\|x-y\|)$ for any $r\ge0$ as:
 $$\tilde\ell(r)=\sup_{x;\|x^*-x\|\leq r} \left[f(x^*) - f(x)\right].$$
 Thus $f(x^*)-\ell(x,x^*)$ naturally forms a lower-envelope of $f$. Thus assuming that the first inequality of $\eqref{eq:enveloppes}$ (upper-envelope) holds, then  $d=0$ again. 

However, although the case $d=0$ is quite general, it does not hold in situations where there is a discrepancy between the upper- and lower-envelopes (Figure~\ref{fig:near-opt-example2}).
\begin{figure}[H]
\begin{center}
\includegraphics[height=3cm, width=\columnwidth]{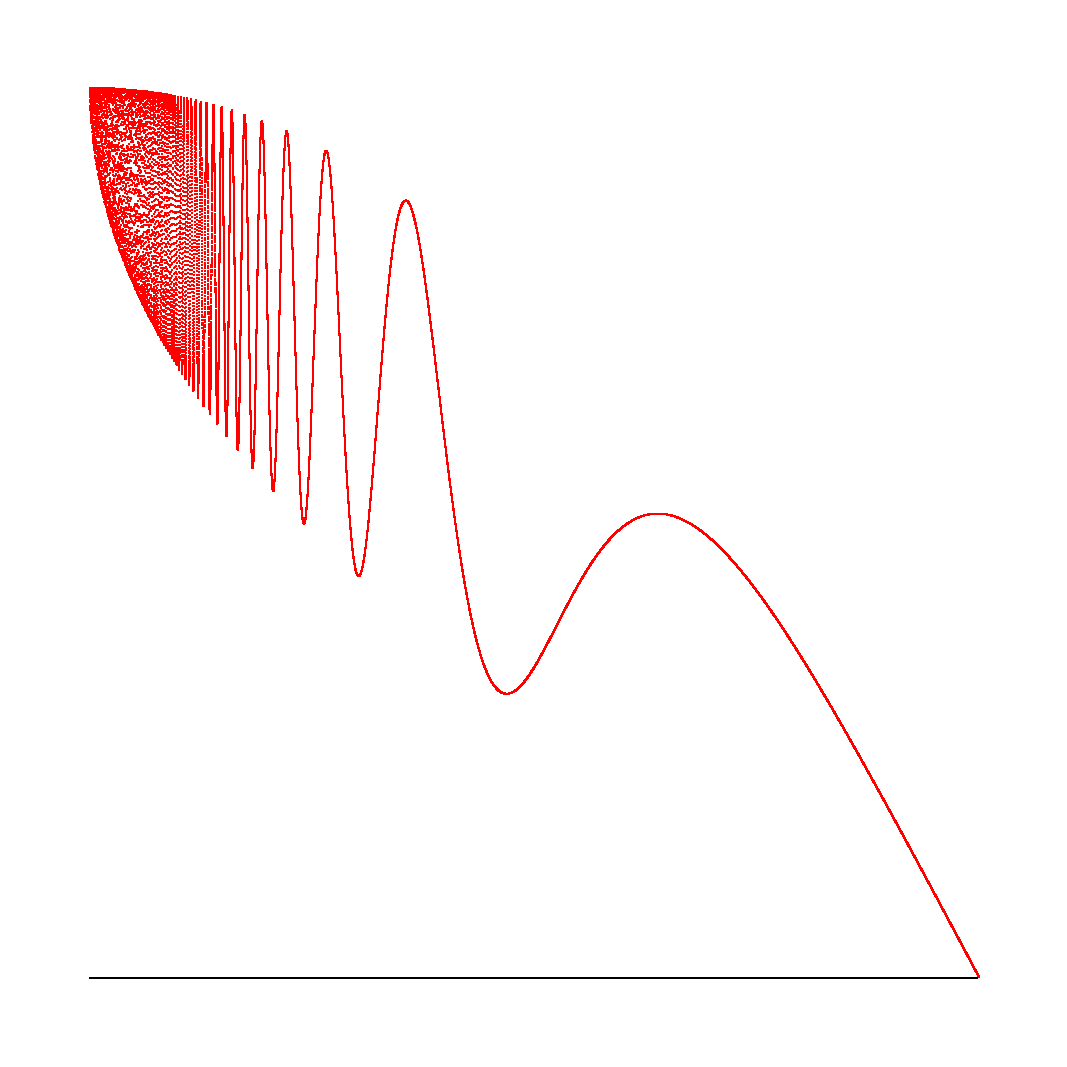}
\caption{We illustrate the case of a function with different order in the upper
and lower envelopes, when $\ell(x,y) =  |x-y|^\alpha$.
 Here $f(x)= 1-\sqrt{x} + (-x^2 +\sqrt{x}
)\cdot(\sin(1/x^2)+1)/2$. The lower-envelope behaves like a square root
 whereas the upper one is quadratic.
The maximum number of $\ell$-balls with radius $\varepsilon$
that can pack ${\cal X}_\varepsilon$
(i.e.,~Euclidean balls with radius $\varepsilon^{1/\alpha}$)
is at most of order $\varepsilon^{1/2}/\varepsilon^{1/\alpha} \leq \varepsilon^{-3/2}$,
since $\alpha\leq 1/2$ in order to satisfy~\eqref{ass:f}. We deduce that there is no semi-metric of the form
$|x-y|^\alpha$ for which
$d<3/2$.\label{fig:near-opt-example2}}
\end{center}
\end{figure}

\section{Experiments}
\label{sec:experiments}
In this section we numerically evaluate the performance of \StoSOO{}\footnote{{\tiny code available at \url{https://sequel.lille.inria.fr/Software/StoSOO}}}.
In all experiments with set the parameters $k$, $\delta$, and 
$h_{\max}$ to the values from Corollary~\ref{thm:col1}.
Moreover, we set the branching factor to $K=3$. 
Note that when the branching factor is an odd number ($K\ge 3$),
we can reuse the evaluations (samples) from the parent node.
Indeed, if $K$ is odd, the representative point of the parent node $\node{h}{i}$
will have the same value as the middle child $\node{h+1}{(K+1)/2}$, 
i.e.,\ $x_{h,i} = x_{h+1,i_{(K+1)/2}}$.
In the case when the domain of $f$ is multi-dimensional,
we only need to split along one dimension at the time,
when expanding the node. In order to preserve bounded diameters
assumption, we can split each time along the dimension in 
which the cell is the largest.

For the evaluation we added a truncated (so that rewards are bounded) 
zero mean Gaussian noise ${\cal N}_T$,
sample of which is shown in Figure~\ref{fig:noise}.
In all the experiments we performed 10 trials and the error bars in the
figures correspond to standard deviations.

\begin{figure}[H]
 \begin{center}
\vspace*{0.3em}
 \includegraphics[width=0.45\columnwidth]{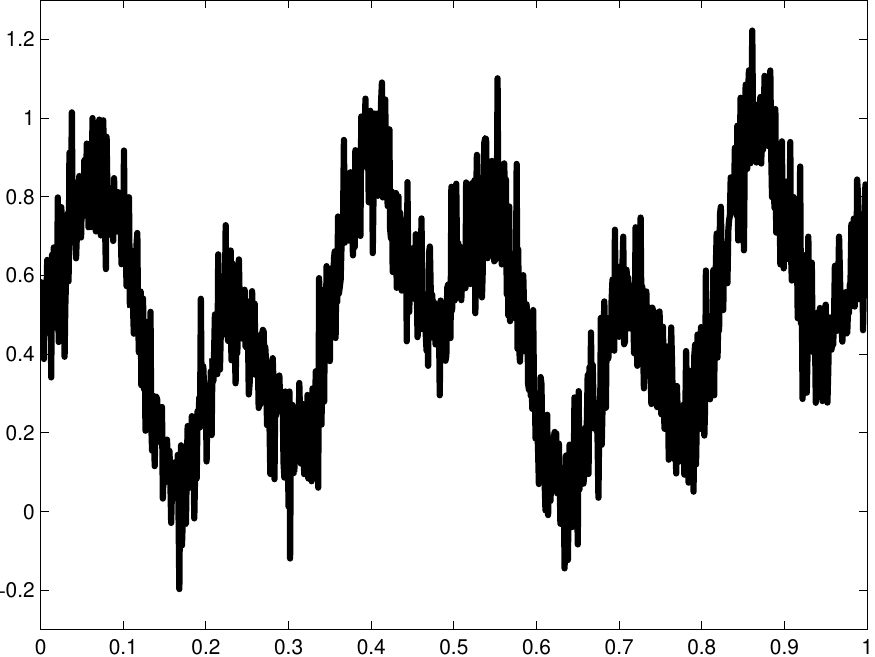}
\quad
\includegraphics[width=0.45\columnwidth]{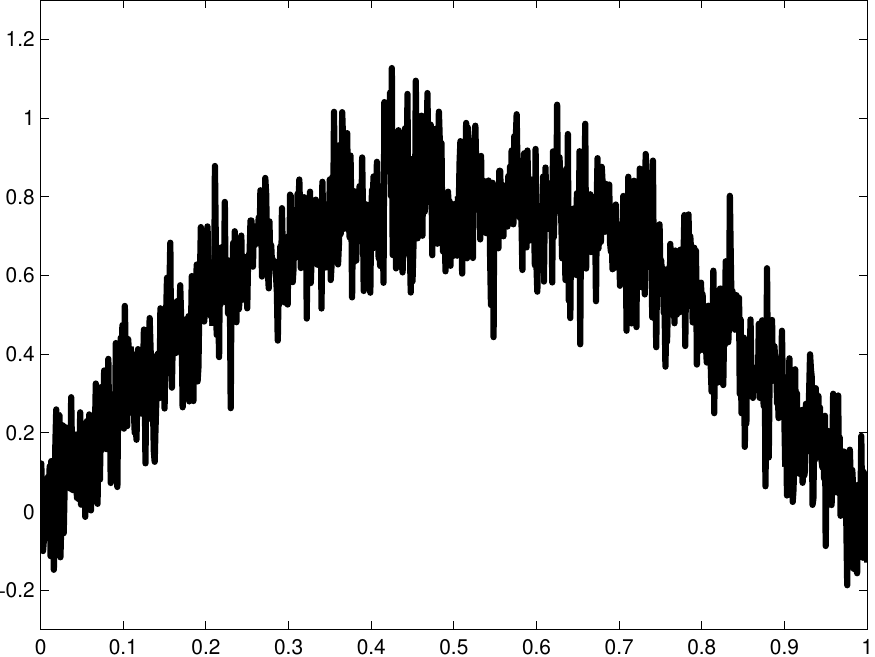}
 \caption{Functions from Figure~\ref{fig:d0functions} noised with ${\cal
N}_T(0,0.1)$.}
\label{fig:noise}
\end{center}
 \end{figure}

\vspace*{-0.3em}
\textbf{Two-sine product:}
In the first set of experiments we consider a two-sine product
function displayed in Figure~\ref{fig:d0functions} (left)
maximized for $f(0.867526) \approx 0.975599$.
Figure~\ref{fig:f1exp} displays the performance of \StoSOO{} for 
different levels of noise. We observe that as we increase 
the number of evaluations, the regret of \StoSOO{} decreases.

\vspace*{0.2em}
 \begin{figure}[H]
 \begin{center}
 \includegraphics[width=0.32\columnwidth]{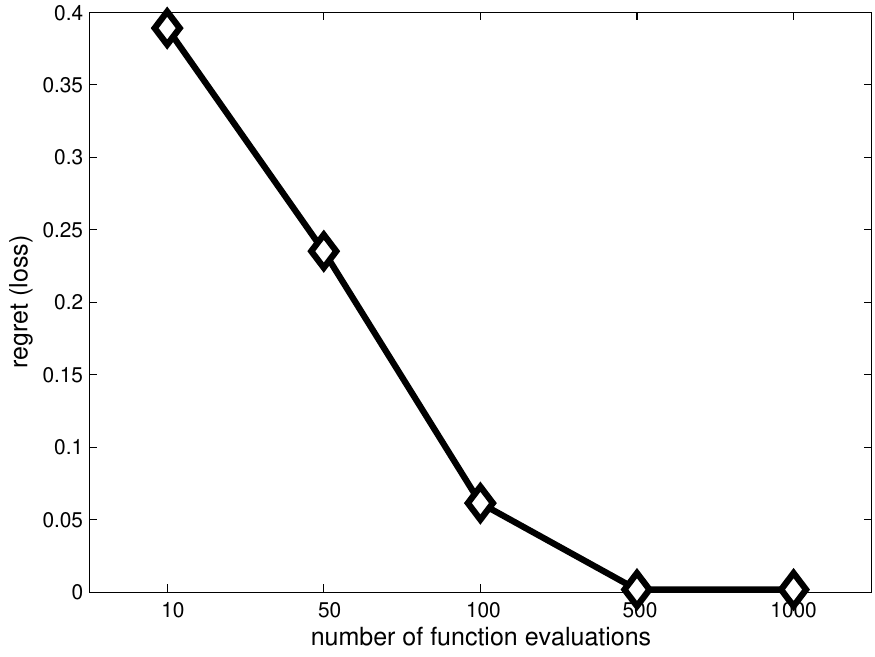} 
\includegraphics[width=0.32\columnwidth]{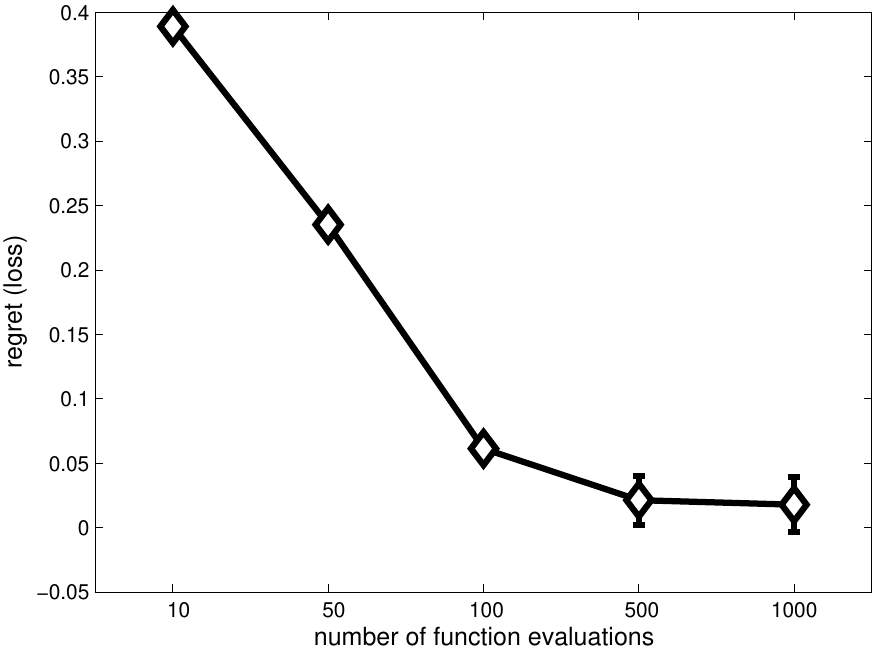} 
\includegraphics[width=0.32\columnwidth]{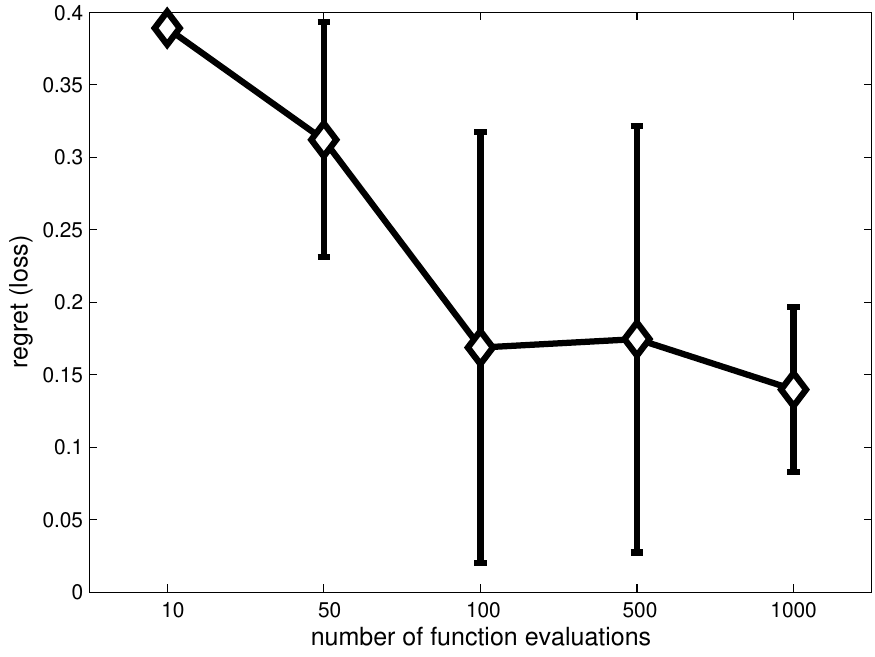} 
\caption{\StoSOO{}'s performance for function $f_1$. 
\textbf{Left}: Noised with  ${\cal N}_T(0,0.01)$. 
\textbf{Middle}: Noised with ${\cal N}_T(0,0.1)$.
\textbf{Right}: Noised with ${\cal N}_T(0,1)$. }
\label{fig:f1exp}
\end{center}
 \end{figure}
\vspace*{-0.2em}

 \begin{wrapfigure}{r}{0.2\textwidth}
 \begin{center}
  \vspace*{-2.7em}
\includegraphics[width=0.4\columnwidth]{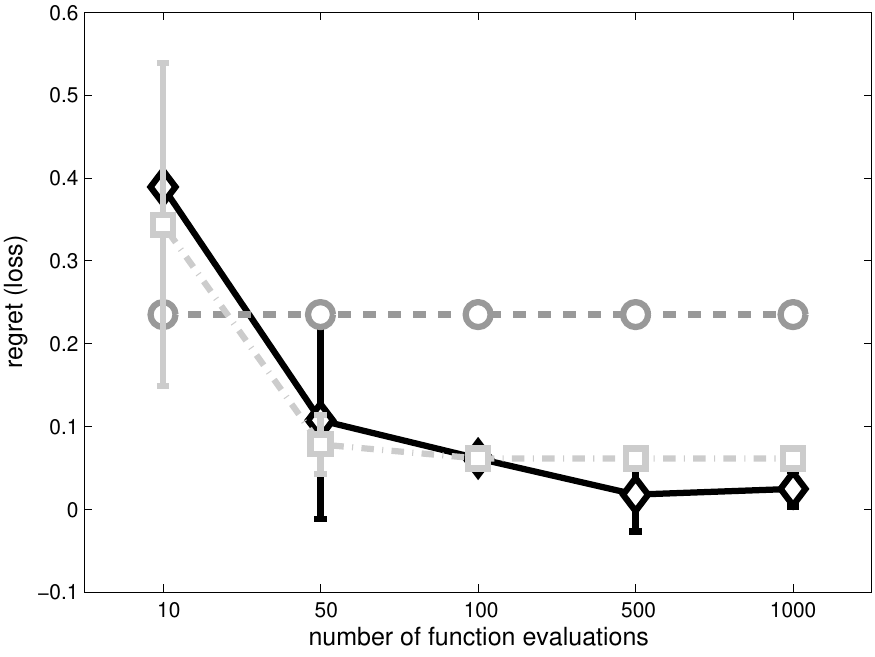}
 \end{center}
\caption{\StoSOO{} (diamonds) vs.\
Stochastic DOO with $\ell_1$ (circles) and $\ell_2$ (squares) on $f_1$.}
\label{fig:comparisons}
 \vspace*{-1em}
\end{wrapfigure}
In Figure~\ref{fig:comparisons}, we compare \StoSOO{}
to the straightforward stochastic version of DOO~\cite{munos2011optimistic},
where we expand each node after $\log(n^2/\delta)/(2\size(h)^2)$ evaluations
(i.e.\  when the size of the confidence interval becomes smaller than the
diameter $\size(h)$ of the cell).
However, (stochastic) DOO needs to know the semi-metric $\ell$
in order to define $\size(h)$.
We evaluate the performance of this stochastic DOO using two semi-metrics that
satisfy Assumption~\hyperref[ass:A1]{A1}: $\ell_1(x,y) = 12|x-y|$ (for which
$d=1/2$)
and $\ell_2(x,y) = 144|x-y|^2$ (for which $d=0$).
We observe that \StoSOO{}
performs as well as stochastic DOO
for the better metric \emph{without} the knowledge of it.

\vspace*{1em}
{\bf Garland function:}
Next, we consider a \textit{garland} function displayed in
Figure~\ref{fig:d0functions} (right).
 The optimization  of this function is challenging because $f_2$ is not
Lipschitz for any $L$. However its near-optimality dimension is still
$d=0$ (Section~\ref{sec:case.d=0.examples}).
Figure~\ref{fig:f2exp} shows the performance of \StoSOO{} as we vary the number of
the evaluations. Notice a higher variance at iteration 200 in the left plot; this is because for that many iterations,
\StoSOO{} was able to reach the depth $h=6$ but only for a few nodes (while only $h=5$ for less iterations)
with small number of $\lceil 200/(\log^3(200)) \rceil = 2$ evaluations.

 \vspace*{1em}

\begin{figure}[H]
 \begin{center}
 \includegraphics[width=0.45\columnwidth]{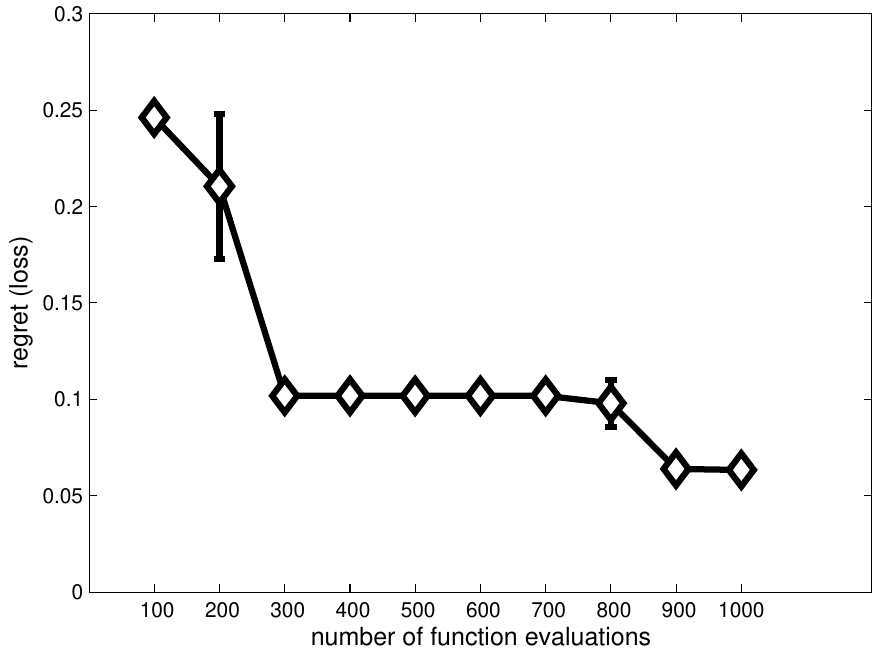} 
\quad
\includegraphics[width=0.45\columnwidth]{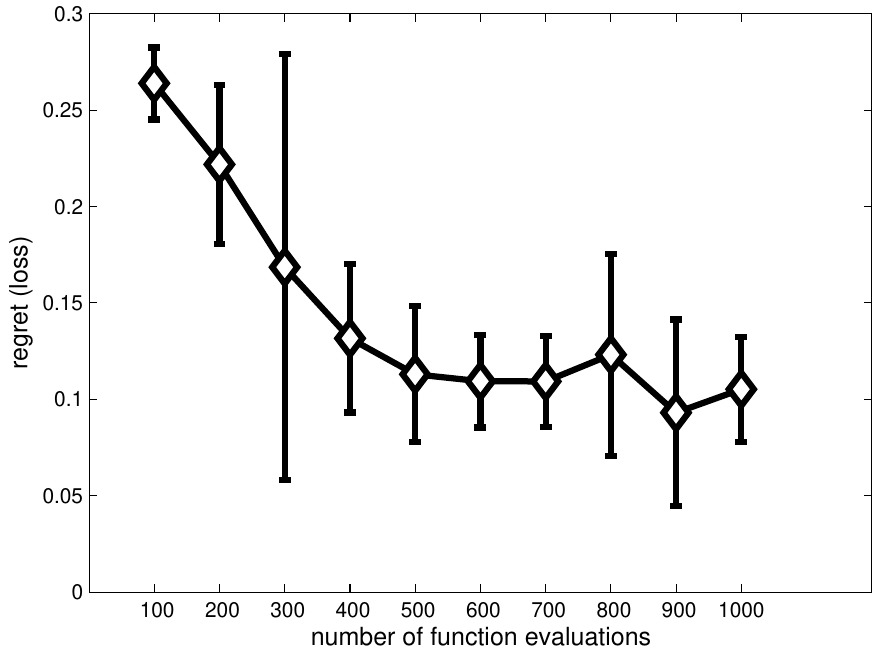} 
\caption{\StoSOO{}'s performance for the garland function.  
\textbf{Left} 
noised with  ${\cal N}_T(0,0.01)$. 
\textbf{Right}: Noised with ${\cal N}_T(0,0.1). $}
\label{fig:f2exp}
 \end{center}
 \end{figure}




\section{Conclusion}
We presented the \StoSOO{} algorithm that is able to optimize black-box stochastic functions, 
without the knowledge of their smoothness.
We derived a finite-time performance bound on the expected loss for the important
case when there exists a semi-metric such that the near-optimality dimension $d=0$.
We showed that this case corresponds to a large class of functions.
In such cases, the performance is almost as good as with an algorithm that would know the best valid semi-metric.
In the future we plan to derive finite-time performance for the case $d>0$.
\section{Acknowledgements}
\label{sec:Acknowledgements}
This research work presented in this paper was supported by European Community's
Seventh Framework Programme (FP7/2007-2013) under grant agreement n$^{\rm o}$
270327 (project CompLACS).


\newpage
\balance
\bibliography{library}
\bibliographystyle{icml2013}

\end{document}